%% file: arxiv.tex
\documentclass[letterpaper, 10 pt, conference]{ieeeconf}  

\IEEEoverridecommandlockouts                              

\include{packages}
\include{macros}
\usepackage{soul}

\title{\Large \bf
  Talk Resource-Efficiently to Me:\\[0.1cm]Optimal Communication Planning for 
  Distributed Loop Closure Detection 
}

\author{Matthew
  Giamou\textsuperscript{$\dagger$}\thanks{\textsuperscript{$\dagger$}Equal
contribution.} \quad Kasra
Khosoussi\textsuperscript{$\dagger$} \quad Jonathan P. How
  \thanks{Authors are
	with the Laboratory for Information and Decision Systems (LIDS), MIT,
	Cambridge, MA --- {\tt\small \{mgiamou,kasra,jhow\}@mit.edu}}%
}

\begin{document}
\maketitle
\thispagestyle{empty}
\pagestyle{empty}

\begin{abstract}
  Due to the distributed nature of cooperative simultaneous localization and
  mapping (CSLAM), detecting inter-robot loop closures necessitates sharing
  sensory data with other robots. A na\"{\i}ve approach to data sharing can
  easily lead to a waste of mission-critical resources. This paper investigates
  the logistical aspects of CSLAM. Particularly, we
  present a general resource-efficient communication planning framework that
  takes into account both the total amount of exchanged data and the induced
  division of labor between the participating robots. Compared to
  other state-of-the-art approaches, our framework is able to
  verify the same set of potential inter-robot loop closures while exchanging
  considerably less data and influencing the induced workloads.
  We develop a fast algorithm for finding globally optimal communication
  policies, and present theoretical analysis to characterize the necessary and sufficient
  conditions under which simpler strategies are optimal. The proposed framework
  is extensively evaluated with data from the KITTI odometry benchmark datasets.
\end{abstract}

\section{Introduction}

Multi-robot, or cooperative, simultaneous localization and mapping
(CSLAM) is an active area of research with a wide spectrum of applications that span from
robotic search and rescue in challenging environments to navigating fleets of autonomous
cars; see \cite{cadena2016past,saeedi2016multiple,indelmandistributed} for 
recent surveys.  Communication is a crucial aspect of the approach, without which
CSLAM would simply reduce to decoupled copies of conventional
SLAM.
In applications
without pre-existing infrastructure, ad-hoc wireless communication is subject
to many shortcomings, including energy constraints, bandwidth, and range
limitations; see, e.g., \cite{saeedi2016multiple,paull2015communication}.
Overlooking these challenges could lead to impractical solutions.
Ensuring that agents are able to effectively and resource-efficiently communicate with one another is 
one of the most challenging problems facing distributed CSLAM
architectures \cite{saeedi2016multiple}. 

Communication is an essential prerequisite for 
establishing loop closures between different robots' trajectories
and maps. To search for
inter-robot loop closures, robots need to compare and match the data acquired throughout 
each of their individual trajectories. However, each
robot initially has access only to the data collected by its own onboard sensors. As a
result, robots need
to frequently share data among themselves. 
State-of-the-art techniques either employ a centralized architecture, or simply
require each robot to broadcast a down-sampled history of its sensory readings; see, e.g.,
\cite{forster2013collaborative} and
\cite{dong2015distributed,indelman2016incremental,indelman2014multi,lazaro2013multi}, respectively.
A na\"{\i}ve approach to the data sharing problem can easily lead to a waste
of mission-critical resources including battery, wireless bandwidth, and CPU
time.

We present a general communication planning framework for resource-efficient data exchange
in the search for inter-robot loop closures in distributed
CSLAM front-ends.
Our framework has several appealing features:

\begin{itemize}
  \item[$\diamond$] A guarantee to be \emph{lossless} in the sense that, for any given  set
	of candidate matches, the proposed framework allows for a \emph{complete
	search}
	of all inter-robot loop closures that exist within that set.
  \item[$\diamond$]  Efficient algorithms for
	finding optimal exchange policies with respect to the total amount of
	data transmission with minimal computational overhead.
  \item[$\diamond$]  Providing a mechanism through which
	one can retain communication efficiency while influencing the final induced
	division of labor between the robots. This
	allows the team to balance the resulting induced workloads based on the
	distribution of computational resources among the robots.
  \item[$\diamond$] Applicability to systems that use measurements and maps composed of any 
  	data type, including dense 3D laser scans and local image features, e.g., BRIEF \cite{calonder2010brief}.
\end{itemize}

\begin{figure*}[t]
  \centering
  \begin{subfigure}[t]{0.18\textwidth}
	\centering
	\begin{tikzpicture}[scale=0.7]
	  \tikzstyle{robot1}=[draw,circle,fill=kkGreen,very thick,scale=1.5,inner
	  sep=0.1cm,outer sep=1pt]
	  \tikzstyle{robot2}=[draw,circle,fill=kkRed,very thick,scale=1.5,inner
	  sep=0.1cm,outer sep=1pt]
	  \node[robot1] at (0,0) (r1) {\small\faAndroid};
	  \node[robot2] at (4,0) (r2) {\small\faAndroid};
	  \DoubleLineOneDirectional{r1}{r2}{Data}{$\Mcal$}
	\end{tikzpicture}
	\caption{}
	\label{fig:unidirectional}
  \end{subfigure}
  \qquad\,\,
  \begin{subfigure}[t]{0.18\textwidth}
	\centering
	\begin{tikzpicture}[scale=0.8]
	  \tikzstyle{robot}=[draw,circle,fill=kkGreen,very thick,scale=1.5,inner sep=0.1cm,outer sep=1pt]
	  \tikzstyle{broker}=[draw,rectangle, very thick,fill=kkBlue,scale=1.2,inner sep=0.25cm,outer sep=3pt]
	  \node[robot] at (0,0) (r1) {\small\faAndroid};
	  \node[robot] at (4,0) (r2) {\small\faAndroid};
	  \node[broker] at (2,2.2) (b) {{\small \faExchange}};
	  \DoubleLine{r1}{r2}{Data}{Data}
	  \DoubleDashedLine{r1}{b}{$\color{black!50}\pi^\star$}{\small$\Mcal_1$}
	  \DoubleDashedLineRev{r2}{b}{\small $\Mcal_2$}{$\color{black!50}\pi^\star$}
	\end{tikzpicture}
	\caption{}
	\label{fig:base}
  \end{subfigure}
  \qquad\,\,\,\,\,\,
  \begin{subfigure}[t]{0.18\textwidth}
	\centering
	\begin{tikzpicture}[scale=0.8]
	  \tikzstyle{robot}=[draw,circle,fill=kkGreen,very thick,scale=1.5,inner sep=0.1cm,outer sep=1pt]
	  \tikzstyle{robot1}=[draw,circle split,fill=kkBlue,very thick,scale=1.1,inner sep=0.1cm,outer sep=1pt]
	  \tikzstyle{broker}=[draw,rounded rectangle, very thick,fill=kkBlue,scale=1.2,inner sep=0.25cm,outer sep=3pt]
	  \node[robot] at (0,0) (r1) {\small\faAndroid};
	  \node[robot] at (4,0) (r2) {\small\faAndroid};
	  \node[robot1] at (2,2.2) (b) 
	  {$\text{\small\faAndroid}$\nodepart{lower}$\text{\small\faExchange}$};
	  \DoubleLine{r1}{r2}{Data}{Data}
	  \DoubleDashedLine{r1}{b}{$\color{black!50}\pi^\star$}{\small$\Mcal_1$}
	  \DoubleDashedLineRev{r2}{b}{\small $\Mcal_2$}{$\color{black!50}\pi^\star$}
	\end{tikzpicture}
	\caption{}
	\label{fig:3rdrobot}
  \end{subfigure}
  \qquad\,\,\,\,\,
  \begin{subfigure}[t]{0.18\textwidth}
	\centering
	\begin{tikzpicture}[scale=0.7]
	  \tikzstyle{robot2}=[draw,circle,fill=kkGreen,very thick,scale=1.5,inner sep=0.1cm,outer sep=1pt]
	  \tikzstyle{robot1}=[draw,circle,fill=kkGreen,very thick,scale=1.5,inner sep=0.1cm,outer sep=1pt]
	  \tikzstyle{broker}=[draw,rounded rectangle, very thick,fill=kkBlue,scale=1.2,inner sep=0.25cm,outer sep=3pt]
	  \node[robot1] at (0,0) (r1)
	  {$\underset{\text{\faExchange}}{\text{\small\faAndroid}}$};
	  \node[robot2] at (4,0) (r2) {\small\faAndroid};
	  \DoubleLine{r1}{r2}{Data}{Data}
	  \xDoubleDashedLine{r1}{r2}{$\color{black!50}\pi^\star$}{\small$\Mcal$}
	\end{tikzpicture}
	\caption{}
	\label{fig:samerobot}
  \end{subfigure}
  \caption{\small An overview of distributed sensory data exchange approaches in CSLAM.
	Figure~\ref{fig:unidirectional} illustrates a monolog (unidirectional
  policy), in which one robot sends all of its scans to the other. In some of
  the state-of-the-art techniques,
this process happens also in the opposite direction (both robots share all of
their scans with each other); see, e.g., \cite{dong2015distributed}. In addition to the sensory data, robots also need to
transmit a smaller amount of information (``metadata'' $\Mcal$) to help identify
potential loop closures (e.g., compact bag-of-words (BoW) vectors for visual place recognition or sparse trajectories). Figures~\ref{fig:base}, \ref{fig:3rdrobot}, and \ref{fig:samerobot} illustrate the proposed approach.
Contrary to Figure~\ref{fig:unidirectional}, here robots engage in a dialog, and
each shares a subset of its sensory data with the other robot. We demonstrate that this process can significantly reduce the total amount of
exchanged data. In our approach, robots still need to exchange metadata. The
broker (\faExchange) then solves the optimal data exchange problem and sends the
optimal exchange policy
$\pi^\star$ for execution to the robots.}
  \label{fig:overview}
\end{figure*}
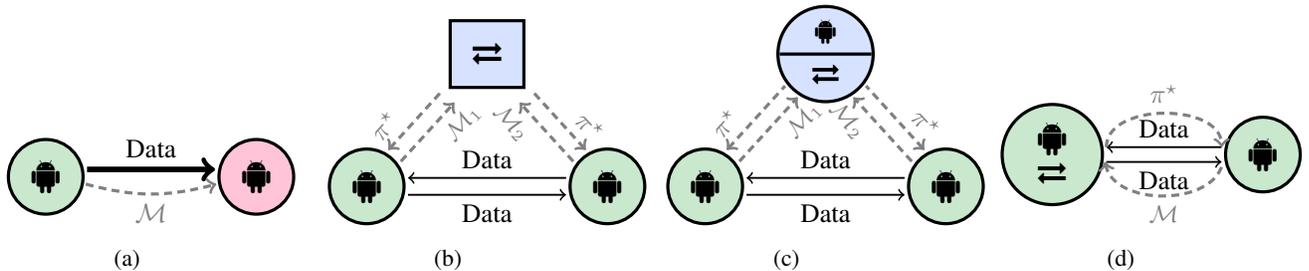

\subsection{Related Work}

In a series of papers, Indelman et al. and Dong et
al.~\cite{dong2015distributed,indelman2014multi,indelman2016incremental} develop
a pose-graph CSLAM framework based on Expectation-Maximization. 
From the perspective of
data exchange efficiency, in
\cite{dong2015distributed,indelman2014multi,indelman2016incremental} robots
broadcast a downsampled subset of their (informative) raw measurements (e.g., laser scans) with
each other. Our work can be employed alongside these and similar systems to provide an
alternative resource-efficient communication plan.

Cieslewski and Scaramuzza~\cite{cieslewski2017efficient} investigate the
scalability of decentralized visual place recognition---in terms of the amount of
exchanged data per place recognition query---in large teams of robots. In particular, they propose a
decentralized approach whose scalability is comparable to that of centralized
architectures and significantly better than the existing decentralized
approaches. In \cite{cieslewski2017efficient} it is empirically shown that their
heuristic approximation only suffers a mild reduction in place recognition recall. The core idea in \cite{cieslewski2017efficient} is to
send partial queries to every other robot, assess the returned image similarity scores, and send the full
query only to the robot with the most likely candidate match.
Unlike the online (frame-query) flavour of the problem addressed in \cite{cieslewski2017efficient},
our work considers a batch formulation that arises in occasional, but
larger, data
exchanges. The batch setting is especially well suited to applications
in which multiple robots are distributed to cover a large space, and
communication is only possible during rendezvous. In such
settings, rendezvous are seen as short-lived valuable opportunities that can be
leveraged to better achieve the mission objective. Furthermore, 
while \cite{cieslewski2017efficient} concerns fleet-wide communication
efficiency in detecting potential inter-robot loop closures, here we focus
on local (pairwise) efficiency in exchanging sensory data. Despite these
differences, an extension of our
framework to $n$-way data exchanges can use the idea behind \cite{cieslewski2017efficient} to improve the
communication efficiency of its metadata
exchange phase (see Section~\ref{sec:OSEP}).

Sharing compressed beliefs and graphs constitutes another type of information exchange that
arises in CSLAM. State-of-the-art techniques often marginalize out unnecessary
intermediate poses from the belief to reduce the amount of exchanged data; see,
e.g., \cite{cunningham2012fully,paull2015communication}. The
resulting information matrix, however, is generally dense. 
This has led to the study of approximate sparisfication
techniques to ``compress'' the reduced beliefs. 
Paull et al.~\cite{paull2015communication} investigate CSLAM with acoustic communication
in the context of autonomous underwater vehicles. They propose a consistent
(conservative) sparsification
scheme based on Kullback-Leibler divergence. 
Lazaro et al.~\cite{lazaro2013multi} propose to 
transmit a reduced representation of robots' graphs (``condensed graphs''), as
well as the most recent laser scans. Sharing only the most recent laser scans
comes at the cost of losing potential loop closures in the regions that robots
 had explored separately prior to the encounter (i.e., before establishing
a communication link).  
Cunningham et al.~\cite{cunningham2012fully} propose a fast RANSAC-based data association
scheme for CSLAM. The communication module in \cite{cunningham2012fully}
shares the reduced beliefs (``condensed maps'') with a bounded number of
robots within communication range.
In contrast to our work, \cite{cunningham2012fully} considers a feature-based
formulation with purely geometric (point) features. In that setting, each
landmark measurement consists of a pair of range and bearing values, which is
typically too lightweight 
to necessitate a communication planning framework.

Forster et
al.~\cite{forster2013collaborative} propose a centralized
framework, in which the base station aggregates all visual information and
establishes inter-robot and intra-robot loop closures; see
\cite{saeedi2016multiple} for more centralized CSLAM approaches. Centralized approaches have
limited applications and, compared to our work, leave no room for communication
efficiency.

Montijano et al.~\cite{montijano2013distributed} and Leonardos et
al.~\cite{leonardos2017distributed} explore elegant formulations and algorithms
for \emph{solving} the distributed data association problem with an emphasis on
maintaining {association consistency} across the communication graph.
Unlike \cite{montijano2013distributed,leonardos2017distributed}, our paper takes a
step back and investigates the logistics of distributed data
association through exchanging data between pairs of robots. Our
approach is orthogonal to such techniques and can be employed alongside 
distributed solvers.

In summary, our framework neither tells the agents \emph{what} to say to each
other---a question that is partly a system-dependant design choice, partly
addressed by belief compression methods, see e.g.,
\cite{paull2015communication,paull2016unified,lazaro2013multi}, and partly
addressed by measurement selection schemes, e.g., see
\cite{kasra16wafr,ila2010information}---nor does
it tell them \emph{what to do} with the exchanged data---i.e., how to solve the
data association or the resulting inference problem, which is addressed by works such as
\cite{montijano2013distributed,leonardos2017distributed,cunningham2012fully,indelman2016incremental}
among others; it rather advises them on \emph{how} to communicate more
effectively and efficiently.  
\subsection{Contribution} This paper addresses the
data exchange problem, a key prerequisite for realizing resource-efficient
distributed inter-robot loop closure detection and place recognition. We formalize the problem, provide a theoretical
analysis, and shed light on its connection to the weighed minimum bipartite
vertex cover problem. These insights ultimately lead to a fast algorithm for
finding globally optimal communication plans based on linear programming.
Additionally, we experimentally validate the proposed framework based on real
benchmark datasets.

\subsection*{Notation}
Bold lower-case and upper-case letters are reserved for vectors and matrices,
respectively.  $\ones$ and $\zero$ denote, respectively, the column vectors of
all ones and all zeros. Sets are shown by upper-case letters.  $|\Acal|$ denotes the
cardinality of set $\Acal$. 
The disjoint set union operator is denoted by $\uplus$ such that $\Acal \uplus
\Bcal = \Acal \cup \Bcal$ and implies that $\Acal \cap \Bcal = \varnothing$.
For any two vertices $u$ and $v$ in a given graph, $u \sim v$ means that there
is an edge connecting $u$ to $v$. Finally, for any set of vertices $\Scal$, $\Ncal(\Scal)$ is the
neighbourhood of $\Scal$ in the graph.

\section{Optimal Data Exchange}
\label{sec:OSEP}

\begin{figure*}[h]
  \centering
  \begin{subfigure}[t]{0.32\textwidth}
	\centering
	\begin{tikzpicture}[scale=1.4]
	  \tikzstyle{vertex}=[circle,fill,scale=0.4,draw]
	  \tikzstyle{special vertex}=[circle,fill=red,scale=0.4,draw]
	  \node[vertex] at (0,0) (a1) {};
	  \node[vertex] at (1,0) (a2) {};
	  \node[vertex] at (2,0) (a3) {};
	  \node[vertex] at (3,0) (a4) {};
	  \node[vertex] at (0,1) (b1) {};
	  \node[vertex] at (1,1) (b2) {};
	  \node[vertex] at (2,1) (b3) {};
	  \node[vertex] at (3,1) (b4) {};
	  \draw[line width=0.3mm, dashed, gray] (-.5,.5) -- (3.53,0.5);
	  \node (r1) at (-.5,0) [] {\faAndroid$_{_1}$};
	  \node (r2) at (-.5,1) [] {\faAndroid$_{_2}$};
	  \draw[thick](a1) -- (b1);
	  \draw[thick](b1) -- (a1);
	  \draw[thick](a1) -- (b2);
	  \draw[thick](a1) -- (b3);
	  \draw[thick](a1) -- (b4);
	  \draw[thick](b1) -- (a2);
	  \draw[thick](b1) -- (a3);
	  \draw[thick](b1) -- (a4);
	\end{tikzpicture}
	\caption{An exchange graph $\Gex$}
	\label{fig:gex}
  \end{subfigure}
  ~
  \begin{subfigure}[t]{0.32\textwidth}
	\centering
	\begin{tikzpicture}[scale=1.4]
	  \tikzstyle{vertex}=[circle,fill,scale=0.4,draw]
	  \tikzstyle{special vertex}=[diamond,fill=red,scale=0.4,draw=red]
	  \node[special vertex] at (0,0) (a1) {};
	  \node[vertex] at (1,0) (a2) {};
	  \node[vertex] at (2,0) (a3) {};
	  \node[vertex] at (3,0) (a4) {};
	  \node[special vertex] at (0,1) (b1) {};
	  \node[vertex] at (1,1) (b2) {};
	  \node[vertex] at (2,1) (b3) {};
	  \node[vertex] at (3,1) (b4) {};
	  \draw[line width=0.3mm, dashed, gray] (-.5,.5) -- (3.53,0.5);
	  \node (r1) at (-.5,0) [] {\faAndroid$_{_1}$};
	  \node (r2) at (-.5,1) [] {\faAndroid$_{_2}$};
	  \draw[ultra thick,->](a1) -- (b1);
	  \draw[ultra thick,->](b1) -- (a1);
	  \draw[thick,->](a1) -- (b2);
	  \draw[thick,->](a1) -- (b3);
	  \draw[thick,->](a1) -- (b4);
	  \draw[thick,->](b1) -- (a2);
	  \draw[thick,->](b1) -- (a3);
	  \draw[thick,->](b1) -- (a4);
	\end{tikzpicture}
	\caption{An admissible policy $\pi$}
	\label{fig:policy}
  \end{subfigure}
  ~
  \begin{subfigure}[t]{0.32\textwidth}
	\centering
	\begin{tikzpicture}[scale=1.4]
	  \tikzstyle{vertex}=[circle,fill,scale=0.4,draw]
	  \tikzstyle{special vertex}=[circle,fill=red,scale=0.4,draw]
	  \node[vertex] at (0,0) (a1) {};
	  \node[vertex] at (1,0) (a2) {};
	  \node[vertex] at (2,0) (a3) {};
	  \node[vertex] at (3,0) (a4) {};
	  \node[vertex] at (0,1) (b1) {};
	  \node[vertex] at (1,1) (b2) {};
	  \node[vertex] at (2,1) (b3) {};
	  \node[vertex] at (3,1) (b4) {};
	  \draw[line width=0.3mm, dashed, gray] (-.5,.5) -- (3.53,0.5);
	  \node (r1) at (-.5,0) [] {\faAndroid$_{_1}$};
	  \node (r2) at (-.5,1) [] {\faAndroid$_{_2}$};
	  \draw[ultra thick,green!50!blue!50!black](a1) -- (b1);
	  \draw[thick,green!50!black](a1) -- (b2);
	  \draw[thick,green!50!black](a1) -- (b3);
	  \draw[thick,green!50!black](a1) -- (b4);
	  \draw[thick,blue!50!black](b1) -- (a2);
	  \draw[thick,blue!50!black](b1) -- (a3);
	  \draw[thick,blue!50!black](b1) -- (a4);
	\end{tikzpicture}
	\caption{The division of labor induced by $\pi$}
	\label{fig:workload}
  \end{subfigure}
  \caption{\small A simple data exchange problem between two robots.  Each
	vertex corresponds to a robot pose, and each edge represents a potential
	loop closure between the corresponding robot poses. There is a ``scan''
	associated with each robot pose. To verify a potential inter-robot loop
	closure between two connected vertices, at least one robot needs to share
	its scan with the other robot. (a) A simple exchange graph. (b) An
	admissible exchange policy in which each robot shares the sensory data
	collected at its red vertex with the other robot. The orientation of each edge signifies the direction of
	exchange (i.e., vertex label). (c) The workload induced by $\pi$:
	\faAndroid$_{_2}$ is responsible for searching for loop closures among the
	green candidates, \faAndroid$_{_1}$ will search among the blue candidates.
  Note that the thick candidate edge will be screened by both robots.}
  \label{fig:bpgraph}
\end{figure*}
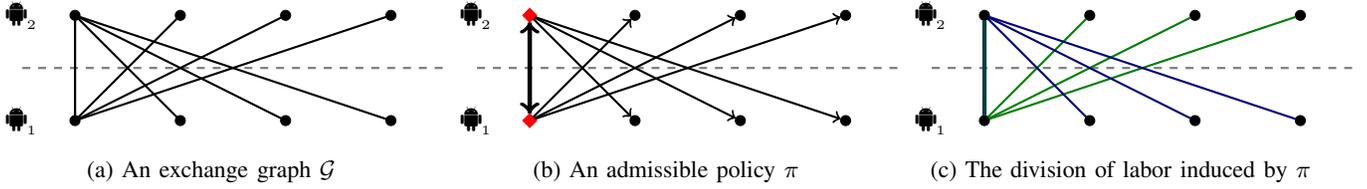
This section proposes a resource-efficient framework to facilitate the search for
inter-robot loop closures in CSLAM via exchanging collected sensory data
(collectively called ``scans'' in this paper).
Each exchange operation is moderated by an \emph{exchange
broker}, which can be a trusted hardware/software component located at one of the
two participating robots (Figure~\ref{fig:samerobot}), or a trusted third party (e.g., another robot or a
base station)---see Figures~\ref{fig:base} and \ref{fig:3rdrobot}. The broker has the duty of initiating, planning, and executing
the operation. Unlike the \emph{servers} in centralized approaches, the exchange
brokers are not
meant to aggregate scans, but rather to advise robots on the ``optimal''
exchange policy.
Although only one broker is needed per exchange process, the total number of
(potential) brokers in a team may vary between $1$ (a central broker) and the number of robots
(each robot can act as a broker if necessary), as long as the
broker is able to communicate with the two participating robots during an
exchange.
\subsection{Initiating an Exchange}
The exchange process can be initiated between two robots when they are within
communication range.
First, the broker has to form the \emph{exchange graph} $\Gex$.
\begin{definition}[Exchange Graph]
  \normalfont
  An exchange graph is an undirected bipartite graph $\Gex = (\VV_1 \uplus
  \VV_2, \Lcal)$  whose
  vertices correspond to the two
  robots' poses involved in the data exchange problem, and $\VV_1 \ni u \sim v \in \VV_2$ iff there is a ``potential'' inter-robot loop closure
  between their corresponding poses. 
\end{definition}
Without loss of generality and for convenience, we assume the degree of each
vertex in the exchange graph is at least one.
$\Lcal$ is a set of
plausible inter-robot loop closure candidates, determined based on geometry
(e.g., trajectory estimates and sensor characteristics such as field of view and
range) and/or appearance (e.g., visual
place recognition systems).
Visual place recognition techniques like the DBoW2 system of 
\cite{GalvezTRO12} can be used to form elements of $\Lcal$ using only information local 
to individual measurements. In the case of DBoW2, this information is vocabulary labels 
of BRIEF \cite{calonder2010brief} features extracted from query images. In both cases,
$\Gex$ is populated without sharing the entirety of the robots' measurement
data. Instead, $\Lcal$ is formed using a compact representation of the sensory
data (hereafter, ``metadata"), e.g., a
collection of bag-of-words (BoW) vectors. Robots cooperate with the broker (by,
e.g., providing information about their 
beliefs over their trajectories or BoW vectors) to form $\Lcal$.  In practice, a
considerable number of potential edges are not plausible given the available
information.
This often makes $\Gex$ far less dense than the complete bipartite graph. The structure of $\Gex$ ultimately 
depends on a variety of factors, including the particular sensors and perception models utilized, the 
level of uncertainty in the robots' beliefs, and perceptual
aliasing.

\subsection{Optimal Data Exchange}
\label{sec:osedef}
The optimal data exchange problem is now formally defined.

\begin{definition}[Data Exchange Policy]
  \normalfont
  A data exchange policy (DEP) is a vertex labeling that specifies which
  ``scans'' should be
  exchanged between a pair of robots. Formally, we call $\pi : \Vcal \to
  \{0,1\}$ a DEP over $\Vcal \triangleq \VV_1 \uplus \VV_2$ in which $\pi(v) = 1$ (resp., $\pi(v) = 0$)
  indicates that the scan collected at vertex $v$ should (resp., should
  \emph{not}) be sent to the other robot.
\end{definition}

Based on the above definition, $\pi$ can be \emph{executed} simply by scanning
the labels and transmitting the scans marked with ``$1$'' (i.e., to be sent); see
Algorithm~\ref{alg:exSEP}.
\begin{algorithm}[t]
  \caption{Execute a DEP}\label{alg:exSEP}
  \begin{algorithmic}[1]
	\For {$v \in \Vcal$}
	\If {$\pi(v) = 1$}
	\State Share $\Scal_v$ (corresponding sensory data).
  \EndIf
\EndFor
\end{algorithmic}
\end{algorithm}

\begin{definition}[Admissible Policy]
  \normalfont
  A DEP is
  called \emph{admissible} iff it allows for a \emph{complete} search; i.e.,
  finding all possible loop closures in $\Lcal$.
  This can be achieved iff, for each edge in the exchange graph, at least
  one robot shares its associated scan with the other robot.
  Formally, $\pi$ is admissible iff for all $u \sim v$, $\pi(u) +
  \pi(v) \geq 1$.
  \label{def:admiss}
\end{definition}

\begin{definition}[Monolog]
  \normalfont
  Let $\VV_\text{source} \in \{\VV_1,\VV_2\}$. An exchange policy $\pi : \VV \to
  \{0,1\}$ is called a
  \emph{monolog} if
  \begin{equation}
	\pi : v \mapsto \begin{cases}
	  1 & v \in \VV_\text{source}, \\
	  0 & \text{otherwise.}
	\end{cases}
	\label{}
  \end{equation}
\end{definition}
\begin{lemma}
  \normalfont
  Every monolog is admissible.
  \label{lemma:uniAdmis}
\end{lemma}

The broker can guarantee the completeness of search by proposing an admissible
policy---but which one of them? Two primitive objectives are considered in this
work:

\noindent 1) \underline{Communication}:
The first objective quantifies the communication cost incurred during the
execution of an exchange policy---mainly due to bandwidth and energy
consumption. The communication cost is measured by the total amount of exchanged
data. From this perspective, $\pi$ is preferred over $\pi^\prime$ iff it
can conduct a complete search by exchanging less data between the two robots.
More precisely, let $w_s: \Vcal \to \mathbb{R}_{\geq 0}$ be a weight function
defined over $\VV$ 
such that $w_s(v)$ quantifies the ``size'' of scan $\Scal_v$ collected at the
corresponding pose. Then, the communication cost incurred as a result of
executing policy $\pi$ can be modelled as
\begin{align}
  \normalfont
  \fcomm(\pi) \triangleq  
  \sum_{v \in \Vcal}  w_s(v)\,\pi(v).
\end{align}
In the special case of uniform weights,
$\fcomm$ reflects the number of exchanges made by $\pi$ (up to a constant).

\noindent 2) \underline{Induced Division of Labor}:
Upon executing an exchange policy, each robot has to perform sensor
registration on a subset of $\Lcal$. The exchange policy implicitly determines
the distribution of the workload between the robots.
The second objective captures this induced
workload. To quantify this workload, first note that any admissible policy $\pi$
divides the initial candidate set into $\Lcal = \Lcal^\pi_1 \cup \Lcal^\pi_2$ in
which $\Lcal^\pi_1$ (resp., $\Lcal^\pi_2$) is the set of edges incident to
$\VV_2$ (resp., $\VV_1$) at a vertex $v$ such that $\pi(v) = 1$. 
These sets can be empty (monolog) and are not necessarily
disjoint: $\Lcal_{12}^\pi \triangleq
\Lcal_1^\pi \cap \Lcal_2^\pi$ is the set of edges like $\{u,v\} \in \Lcal$ such that
$\pi(u)=\pi(v) = 1$ (see Figure~\ref{fig:policy}). 
$\Lcal_1^\pi\setminus\Lcal_2^\pi$ (resp.,
$\Lcal_2^\pi\setminus\Lcal_1^\pi$) can only be searched by the first (resp.,
second) robot. On the contrary, in principle both robots can screen
the candidates in $\Lcal_{12}^\pi$. We can either divide the burden of searching in
$\Lcal_{12}^\pi$ between the robots, or simply let each robot screen it on its
own. The latter is preferred due to the following advantages. First, from a robustness
perspective, verifying $\Lcal_{12}^\pi$ separately on each robot creates a
desirable redundancy in case robots are unable to exchange their newly
discovered loop closures due to problems like communication failure. Furthermore, the cost of post-exchange communication
will be slightly reduced since we do not need to
exchange the loop closures found in $\Lcal_{12}^\pi$
(Section~\ref{sec:postexchange}). Finally, as we will see
shortly, this choice leads to tractable
optimization problems. 

Suppose the computational cost of verifying candidate inter-robot loop
closure $\{u,v\}$ is quantified by $c_{uv} \geq 0$.
The total computational cost due to sensor registration induced by
exchange policy $\pi$ on robot $i \in \{1,2\}$ 
is given by
\begin{align}
  \ell_i^\pi & = 
  \sum_{\mathclap{v \in \VV \setminus \VV_i}} \,\,\,\,\sum_{u \sim v} c_{uv} \pi(v).
  \label{eq:ell}
\end{align}
Note that under uniform $\{c_{uv}\}_{u \sim v}$, $\ell_i^\pi$ is proportional to the
number of potential loop closures that must be verified by robot $i$ as a result
of exchange policy $\pi$.
Let $\alpha_1$ and $\alpha_2$ be non-negative parameters that
control the induced workload balance between the two
robots, such that, e.g., increasing $\alpha_i$ will shift the balance in favor of robot
$i$. For example, in a heterogeneous data exchange between a typical robot and a tactical
supercomputer, one may seek to choose an
admissible policy such that most of the induced workload is redirected toward the 
tactical supercomputer. This narrative results in
\begin{align}
  \fcomp(\pi;\alpha_1,\alpha_2)
  & \triangleq \alpha_1 \ell_1^\pi + \alpha_2 \ell_2^\pi  \\ & =  \sum_{v\in\VV}
  w_\ell(v) \, \pi(v),
  \label{eq:fcomp}
\end{align}
in which
\begin{equation}
  w_\ell: v \mapsto
  \begin{cases}
	\alpha_2 \sum_{u \sim v} c_{uv} & v \in \VV_1,\\ 
	\alpha_1 \sum_{u \sim v} c_{uv} & v \in \VV_2.
  \end{cases}
  \label{eq:wlw}
\end{equation}

\begin{problem}[Optimal Data Exchange Problems (ODEP)]
  \normalfont
  \begin{equation}
	\normalfont
	\begin{aligned}
	  & \underset{\pi}{\text{minimize}}
	  & & \fcomp(\pi;\alpha_1,\alpha_2) \\
	  & \text{subject to}
	  && \text{$\pi$ is admissible.}
	\end{aligned}
	\label{}
	\tag{P$_1$}
  \end{equation}
  \rule{\linewidth}{1pt}
  \begin{equation}
	\normalfont
	\begin{aligned}
	  & \underset{\pi}{\text{minimize}}
	  & & \fcomm(\pi) \\
	  & \text{subject to}
	  && \text{$\pi$ is admissible.}
	\end{aligned}
	\label{}
	\tag{P$_2$}
  \end{equation}
  \rule{\linewidth}{1pt}
  \begin{equation}
	\normalfont
	\begin{aligned}
	  & \underset{\pi}{\text{minimize}}
	  & & f_\bullet(\pi;\alpha_1,\alpha_2,\omega) \\
	  & \text{subject to}
	  && \text{$\pi$ is admissible.}
	\end{aligned}
	\label{}
	\tag{P$_3$}
  \end{equation}
  \begin{align}
	f_\bullet(\pi;\alpha_1,\alpha_2,\omega) & \triangleq \fcomm(\pi) +
	\omega \fcomp(\pi;\alpha_1,\alpha_2) \\
	& = \sum_{v\in\VV} w_\bullet(v) \, \pi(v),
	\label{}
  \end{align}
  in which
  $w_\bullet : v \mapsto  w_s(v) + \omega \, w_\ell(v)$.
  \label{prob:osep}
\end{problem}

\subsection{Solving the Optimal Data Exchange Problem}

It is easy to see that
	  P$_{1:3}$ are all instances of the weighted minimum bipartite vertex cover
	  problem.\footnote{Finding a subset of vertices in a vertex-weighted bipartite
	  graph with the minimum sum of vertex weights such that it covers
	every edge.}
  To see this, first note that the admissibility constraint needed for guaranteeing the
  completeness of search is identical to the constraint
  in vertex cover.
  Translating an instance of one of these narratives to an equivalent instance of
  the other (i.e., mapping a lossless exchange policy to an equivalent vertex cover $\pi
  \mapsto \Pi$ and vice versa) is trivial:
	$\Pi \triangleq \big\{v \in \Vcal : \pi(v) =
	1\big\}$ and $\pi : v \mapsto \mathds{1}_{\Pi}(v)$ where
  \begin{align}
	\mathds{1}_{\Pi}(v) \triangleq
   \begin{cases}
	 1 & \text{if $v \in \Pi$,} \\ 
	 0 & \text{if $v \in \VV\setminus\Pi$.}
   \end{cases}
   \label{}
 \end{align}
 Finally, note that the cost of $\pi$ (in ODEP) is equal to that of 
 $\Pi$ in the weighted minimum bipartite vertex cover, and vice versa.
Consequently P$_{1:3}$ can all be solved
using the same machinery. Furthermore, this result characterizes 
the communication cost incurred in the
search for inter-robot loop closures and the induced workload balance in terms
of the graph topology and
vertex/edge weights through a well-understood graph
invariant. 

\subsection*{Algorithm}
Although the weighted minimum vertex cover problem is NP-hard in general,
it can be solved efficiently in bipartite graphs; see, e.g.,
\cite{schrijver2003combinatorial}.  Therefore, by virtue of
the abovementioned observation, we can solve any ODEP efficiently by
casting it as a weighted minimum bipartite vertex cover problem.
Moreover, Algorithm~\ref{alg:exSEP} can be slightly
restructured to execute the vertex cover translation of an optimal
policy---see Algorithm~\ref{alg:exVC}.
It remains to describe an algorithm based on linear programming (LP) for efficiently
solving ODEP.
Let $w \in \{w_\ell, w_s, w_\bullet\}$.
The corresponding ODEP can then be formulated as the following integer linear program (ILP):
\begin{equation}
  \begin{aligned}
	& \underset{\pi}{\text{minimize}}
	& & \sum_{v \in \Vcal} w(v)\,\pi(v)\\
	& \text{subject to}
	&& \pi(u) + \pi(v) \geq 1 \qquad u \sim v, \\
	&&& \pi(u) \in \{0,1\} \qquad\,\,\,\,\,\,\,\,\, u \in \Vcal.
  \end{aligned}
  \label{prob:ilp}
  \tag{P$_\text{ILP}$}
\end{equation}
The admissibility constraint in~\ref{prob:ilp} can be compactly written as
$\AAAtop\ppp \geq \ones$, in which $\AAA$ is the unoriented
incidence matrix of the exchange graph, and $\ppp$ is the stacked vector
of values $\pi(u)$ for $u \in \Vcal$. Let $\mathbf{w}$ be the stacked vector of
vertex weights.~\ref{prob:ilp} admits a natural LP relaxation by expanding its
feasible set $\Fcal_\text{ILP}$ into $\Fcal_\text{LP} \triangleq \{\ppp :
  \AAAtop\ppp \geq
\ones, \ppp\geq\zero\} \supset \Fcal_\text{ILP}$:
\begin{equation}
  \begin{aligned}
	& \underset{\ppp}{\text{minimize}}
	& & \mathbf{w}^\top{\hspace{-0.05cm}}\ppp\\
	& \text{subject to}
	&& \ppp \in \Fcal_\text{LP}.
  \end{aligned}
  \label{prob:lp}
  \tag{P$_\text{LP}$}
\end{equation}
 It is well known that $\AAA$ is \emph{totally unimodular}, and therefore
 $\Fcal_\text{LP}$ is integral; i.e., \ref{prob:lp} has an
 integral solution that can be found using the simplex algorithm (see, e.g.,
 \cite[Ch.~18]{schrijver2003combinatorial}). Any integral solution corresponds
 to an optimal exchange policy for Problem~\ref{prob:osep}. 
 In the special case of uniform weights, we can construct the optimal policy
 directly from the maximum bipartite matching in $\Gex$; see 
 K\"{o}nig's theorem \cite{schrijver2003combinatorial}.

\subsection*{Optimality Conditions for Monologs}
ODEP is built on the presumption that exploiting bidirectional
communication can lead to more resource-efficient strategies. 
While this is generally true, in some special cases, monologs may perform
optimally as well. The following theorem offers the necessary and sufficient condition for the most general
form of P$_{1:3}$ under which a monolog is optimal.

\begin{theorem}
  \normalfont
  Consider a vertex-weighted exchange graph $\Gex$ with non-negative weights
  assigned by $w : \VV \to
  \Rset_{\geq 0}$. Let $\VV_\circ \in \{\VV_1,\VV_2\}$.
  The monolog $\pi$ defined as
  \begin{align}
	\pi_\circ : v \mapsto \begin{cases}
	  1 & v \in \VV_\circ,\\
	  0 & \text{otherwise},
	\end{cases}
	\label{}
  \end{align}
  minimizes the cost function
  \begin{align}
	f(\pi) \triangleq \sum_{v \in \VV} w(v) \, \pi(v)
	\label{}
  \end{align}
  among all admissible policies if and only if $\Gex$ satisfies what we call 
  the \emph{generalized Hall's condition} (GHC): 
  \begin{equation}
	\forall \Scal\subseteq \VV_\circ: \qquad \sum_{v\in\Scal} w(v) \leq
	\sum_{\mathclap{v \in \Ncal(\Scal)}} w(v).
	\tag{GHC}
  \end{equation}
  \label{th:generalized}
\end{theorem}
\begin{proof}
  \normalfont
  \noindent [$\Rightarrow$]
  We show the contrapositive.
  Suppose there exists a $\Scal \subseteq \VV_\circ$ that violates 
  GHC. Consider,
  \begin{align}
	{\pi}^\ast : v \mapsto 
	\begin{cases}
	  1 & v \in (\VV_\circ \setminus \Scal) \uplus \Ncal(\Scal) \\
	  0 & \text{otherwise}.
	\end{cases}
	\label{}
  \end{align}
  $\pi^\ast$ is admissible since the vertices in $\VV_\circ \setminus \Scal$ cover the edges
  that are not incident to $\Scal$, while those in $\Ncal(\Scal)$ cover every edge incident to
  $\Scal$. Now since $\Scal$ violates GHC we have,
  \begin{align}
	f(\pi^\ast) & = \sum_{\mathclap{v \in \VV_\circ \setminus \Scal}} w(v) +
	\sum_{\mathclap{v\in\Ncal(\Scal)}} w(v) \\ 
	& < \sum_{\mathclap{v \in \VV_\circ \setminus \Scal}} w(v) +
	\sum_{\mathclap{v\in\Scal}} w(v) \\
	& = \sum_{v \in \VV_\circ} w(v) \\
	& = f(\pi_\circ).
	\label{}
  \end{align}
  \noindent [$\Leftarrow$] Now we show GHC is sufficient. Suppose GHC holds and
  let $\pi^\star$ be the optimal admissible policy.
  For simplicity and without loss of generality let us assume $\VV_\circ = \VV_1$. 
  Define $\Pi^\star \triangleq
  \{v \in \Vcal : \pi^\star(v) = 1\}$ and $\Pi^\star_{i} \triangleq \Pi^\star \cap
  \VV_i$ ($i=1,2$). 
  If $\Pi^\star_2$ is empty, $\pi^\star = \pi_\circ$. Furthermore, based on GHC, $\Pi^\star_1$
  cannot be empty unless $\VV_1$ and $\VV_2$ have equal costs, which also
  implies that $\pi^\star = \pi_\circ$.
  Thus we can assume both are non-empty.
  Since $\pi^\star$ is admissible, there must be no edges between
  $\VV_1 \setminus \Pi^\star_1$ and $\VV_2 \setminus \Pi^\star_2$. Therefore, 
  $\Ncal(\VV_1\setminus\Pi_{1}^\star) \subseteq \Pi^\star_2$.
  From GHC and the fact that vertex weights are non-negative we have,
  \begin{align}
	\sum_{\mathclap{v \in \VV_1\setminus\Pi_1^\star}} w(v) & \leq 
	\sum_{\mathclap{v \in \Ncal(\VV_1\setminus\Pi_1^\star)}} w(v)  \\
	& \leq 
	\sum_{\mathclap{v \in \Pi^\star_2}} w(v).
	\label{}
  \end{align}
  Consequently,
  \begin{align}
	f(\pi^\star) & =  \sum_{\mathclap{v \in \Pi^\star_1}}w(v) +
	\sum_{\mathclap{v \in \Pi^\star_2}} w(v) \\
	& \geq  \sum_{\mathclap{v \in \Pi^\star_1}} w(v) +
	\sum_{\mathclap{\VV_1\setminus\Pi^\star_1}} w(v) \\
	& =  \sum_{\mathclap{v\in \VV_\circ}}w(v) \\
	& = f(\pi_\circ).
	\label{}
  \end{align}
  This concludes the proof.
\end{proof}

\begin{remark}
  \normalfont
  Theorem~\ref{th:generalized} states that 
  $\pi_\circ$ is optimal iff, for any subset of
  vertices in $\VV_\circ$, the amount of data that needs to be transmitted from
  $\VV_\circ$ to the other robot is not
  greater than the amount of data needs to be transmitted in the opposite
  direction. Although this
  result is intuitive, the fact that 
  the GHC is both necessary and sufficient is non-trivial.
\end{remark}

\begin{corollary}
  \normalfont
  Let $\VV_\text{max} \in \{\VV_1,\VV_2\}$ be the vertex set with the larger
  $\alpha_i$.
  The monolog $\pi_1$ defined as
  \begin{align}
	\pi_1 : v \mapsto \begin{cases}
	  1 & v \in \VV_{\text{max}} , \\
	  0 & \text{\normalfont otherwise.}
	\end{cases}
	\label{}
  \end{align}
  is optimal with respect to {\normalfont P$_1$}.
  \label{prop:p1}
\end{corollary}
Corollary~\ref{prop:p1} implies that P$_1$ always has a trivial optimal monolog
solution. Nonetheless note that P$_3$ still allows us influence the induced division of
labor based while retaining communication efficiency.
Moreover, Corollary~\ref{prop:p1} also implies that the two
objective functions $\fcomp$ and $\fcomm$ blended together in P$_3$ are
competing with
each other to shift the structure of the optimal policy towards monologs (ideal workload balance)
and dialogs (communication efficiency), respectively.
\begin{corollary}
  \normalfont
  Let $\VV_\text{min} \in \{\VV_1,\VV_2\}$ be the vertex set with smaller
  cardinality.
  The monolog $\pi_2$ defined as
  \begin{align}
	\pi_2: v \mapsto 
	\begin{cases}
	  1 & v \in  \VV_\text{min},\\
	  0 & \text{otherwise,}
	\end{cases}
	\label{}
  \end{align}
  is optimal with respect to P$_2$ under uniform weights iff $\Gex$ satisfies
  Hall's condition (HC): 
	$\forall \Scal\subseteq \VV_\text{min}: |\Scal| \leq |\Ncal(\Scal)|$.
  \label{prop:perfectmatch}
\end{corollary}

Corollary~\ref{prop:perfectmatch} states the necessary and sufficient
condition under which the monolog $\pi_2$ is optimal.
This result also follows directly from Hall's marriage theorem and
K\"{o}nig's theorem \cite{schrijver2003combinatorial}.
As an example, consider the case of \mbox{$k$-regular} bipartite
graphs.\footnote{A graph is called $k$-regular if all of its
vertices have degree $k \geq 1$.}
A well-known application of Hall's marriage theorem implies that
$k$-regular bipartite graphs satisfy HC \cite{schrijver2003combinatorial}.
Similarly, it is easy to check that HC holds in the complete bipartite graph.
Corollary~\ref{cor:kreg} follows from this result and Corollary~\ref{prop:perfectmatch}.
\begin{corollary}
  \normalfont
  The monolog $\pi_2$ is optimal with respect to P$_2$ under uniform weights
  in  $k$-regular, and in complete bipartite graphs.
  \label{cor:kreg}
\end{corollary}

\begin{algorithm}[t]
  \caption{Execute a DEP via Vertex Cover $\Pi$}\label{alg:exVC}
  \begin{algorithmic}[1]
	\For {$v \in \Pi$}
	\State Send $\Scal_v$ to the other robot.
  \EndFor
\end{algorithmic}
\end{algorithm}

\subsection{Post-Exchange Protocol}
\label{sec:postexchange}
After executing the optimal policy
$\pi^\star$, each robot has to verify the potential loop closures in a subset of $\Lcal$
($\Lcal_1^\pi$ and $\Lcal_2^\pi$; see Section~\ref{sec:osedef}) via sensor
registration. Examining the candidates will lead to a set of inter-robot loop
closures $\Lcal^\boxplus \subseteq \Lcal$. Because of the admissibility
constraint, we know that $\Lcal_1^\boxplus \cup \Lcal_2^\boxplus =
\Lcal^\boxplus$ in which $\Lcal^\boxplus_i$ is the set of loop closures
discovered by robot $i \in \{1,2\}$ after executing an admissible exchange
policy (i.e., the search is guaranteed to be complete). At this point, each
robot is aware of its own set of newly discovered inter-robot loop closures;
these sets will have a non-empty overlap iff $\Lcal_{12}^{\pi^\star} \cap
\Lcal^\boxplus$ is non-empty. If the communication channel is still available,
robots can immediately share their newly discovered positive matches with each
other by transmitting $\Lcal^\boxplus_i\setminus\Lcal^\boxplus_{12}$ ($i=1,2$).
The exchange process ends here. At this stage, robots are able to closely examine
every potential candidate, perform geometric verification, solve the sensor registration and data
association problems, and establish relative measurements; see, e.g.,
\cite{montijano2013distributed,leonardos2017distributed,indelman2014multi,mur2015orb}.

\subsection{Exchange Inertia and Dynamic Pricing} 
In Problem~\ref{prob:osep}, vertex weights quantify quantities such as the size of a
scan, computational cost of sensor registration for the corresponding potential
loop closures, and the desired workload balance.
From a broader perspective, the weights can be interpreted
as the \emph{exchange inertia}, such that a smaller weight signifies more
desire to share the associated scan with other robots, and vice versa.
This broader interpretation allows us to incorporate a wider spectrum of
objectives and constraints using the same underlying framework.
In particular, robots and/or the broker may utilize a dynamic pricing strategy
driven by various internal/external incentives.
For example, these dynamic pricing schemes may depend on the specific role of a
robot in the team, its capabilities, clearance
level, privacy restrictions, and the available mission-critical resources.

\begin{algorithm}[t]
  \caption{Optimal Data Exchange}\label{alg:OSEP}
  \begin{algorithmic}[1]
	\State \textcolor{green!40!black}{Robots}: Send the essential metadata to the broker
	\State \textcolor{blue!40!black}{Broker}: Form $\Gex$  (w/ dynamic pricing)
	\State \textcolor{blue!40!black}{Broker}: Form and solve ODEP via LP relaxation
	\State \textcolor{green!40!black}{Robots}: Execute $\pi^\star$ --- exchange scans
	\State \textcolor{green!40!black}{Robots}: Search for loop closures in $\Lcal^\pi_1$ and $\Lcal^\pi_2$
	\State \textcolor{green!40!black}{Robots}: Exchange the discovered loop closures: $\Lcal^\boxplus_i\setminus\Lcal^\boxplus_{12}$
  \end{algorithmic}
\end{algorithm}
\section{Experiments}
\label{sec:results}
Algorithm~\ref{alg:OSEP} summarises the entire ODEP process.
This section presents results obtained using  the KITTI dataset 
\cite{Geiger2013IJRR} to formulate realistic ODEP instances. KITTI was 
chosen for its long, data-rich trajectories, and accurate ground truth.
ODEP instances are solved with the Gurobi LP
solver.\footnote{\url{http://www.gurobi.com}}
Solving ODEP takes about $0.41$ seconds in one of the largest exchange graphs encountered in our datasets
(with more than $2 \times 10^3$ vertices and $96 \times 10^{3}$ edges) on an Intel Core i7-6820HQ CPU operating at 2.70
GHz. The runtime in realistic settings and using DBoW2 with
$\alpha = 0.3$ (see Section~\ref{sec:dbow2}) is about $0.03$ seconds. Due to
space limitation, in this section we focus mainly on P$_2$.

\subsection{Trajectory Geometry Experiments}
In order to create instances of ODEP with the KITTI dataset, we chose sequences of 
the odometry benchmark that contained considerable amounts of self-intersection 
and re-tracing in their ground truth trajectory. Each sequence is divided into two parts corresponding to two distinct robots. For each pose in the trajectory, Oriented FAST and Rotated BRIEF (ORB) features \cite{rublee2011orb} 
exceeding a variable FAST detection threshold are extracted from the associated color 
camera image. Since this set of features can be used to detect and compute loop closures
between poses as part of a SLAM system \cite{mur2015orb}, the number of extracted 
features determines the vertex weight $w_s(v)$ for the pose at vertex $v$. In regions 
with greater environmental detail, a greater number of ORB features are extracted. The 
KITTI dataset's odometry ground truth is then used to form edges between nearby poses associated with each robot. This process results in an exchange graph $\Gex$ with weights 
$w_s(v)$ that depends on a number of parameters:
\begin{enumerate}
\item FAST threshold $k_{F}$ used to detect ORB features,
\item Data rate or measurement frequency $f$ (KITTI data is provided at 10 Hz),
\item Maximum distance $d_\text{max}$ between poses that are candidate matches (i.e. $ u \sim v$),
\item Minimum fraction $\eta$ of range limited camera field of view (FOV) between poses that are candidate matches.
\end{enumerate}
Varying these parameters leads to different structures in $\Gex$ and variable communication savings
when using ODEP. In practice, different sensors and varying confidence in robot trajectory
estimates would permit empirical modelling of exchange graph formation. In this paper, we 
analyze ranges of the above parameters to capture a variety of problem instances. For example,
large values of $d_\text{max}$ correspond to scenarios where each robot's trajectory estimate is 
highly uncertain and, therefore, a greater range of nearby poses need to be considered
loop-closure candidates. Figure \ref{fig:kitti_trajectory} displays edges of $\Lcal$ in green for 
a particular set of parameters on KITTI odometry sequence 0 and sequence 6. Figure \ref{fig:radius_seq0} and \ref{fig:radius_seq6} display the 
communication savings of the optimal policy relative to monolog policies for sequences 0 and 6 when $\Lcal$ is formed between poses within a variable 
$d_\text{max}$. Figures \ref{fig:fov_seq0} and \ref{fig:fov_seq6} report similar
results when $\Lcal$ is formed with a variable minimum FOV overlap $\eta$. The abrupt jumps in cost seen in Figures 
\ref{fig:radius_seq6} and \ref{fig:fov_seq6} are caused by sequence 6's particular trajectory. Figure 
\ref{fig:kitti_trajectory_06} displays the simple elongated loop that sequence 6 follows, along with some 
candidate edges formed by the field of view threshold of $\eta$ = 0.4. These settings lead to candidate edges 
across the thin loop which vanish for shorter values of $d_{\text{max}}$ and higher values of $\eta$,
reducing the required communication cost. 

Figure \ref{fig:kitti_params} demonstrates that solving ODEP enables the robots to reduce the amount of data to 
be exchanged by up to 5 MB over some monologs. Note that in some of the
state-of-the-art systems, full bidirectional communication 
of measurements is utilized by default, resulting in at least the sum of the communication costs of both 
monologs (red and green curves) in Figures \ref{fig:radius_seq0}-\ref{fig:dbow_seq6} 
\cite{dong2015distributed}. For a typical 11 Mb/s ad hoc WiFi network
tested in our laboratory, 5 MB corresponds to approximately 5 seconds of
transmission time. Thus, in addition to reducing use of network bandwidth and
battery usage, communication reduction could potentially help to significantly shorten robot rendezvous 
periods in time-critical missions.
 
\begin{figure}[t!]
\captionsetup[subfigure]{justification=centering}
\begin{subfigure}{0.5\textwidth}
  \centering
  \includegraphics[scale=0.205]{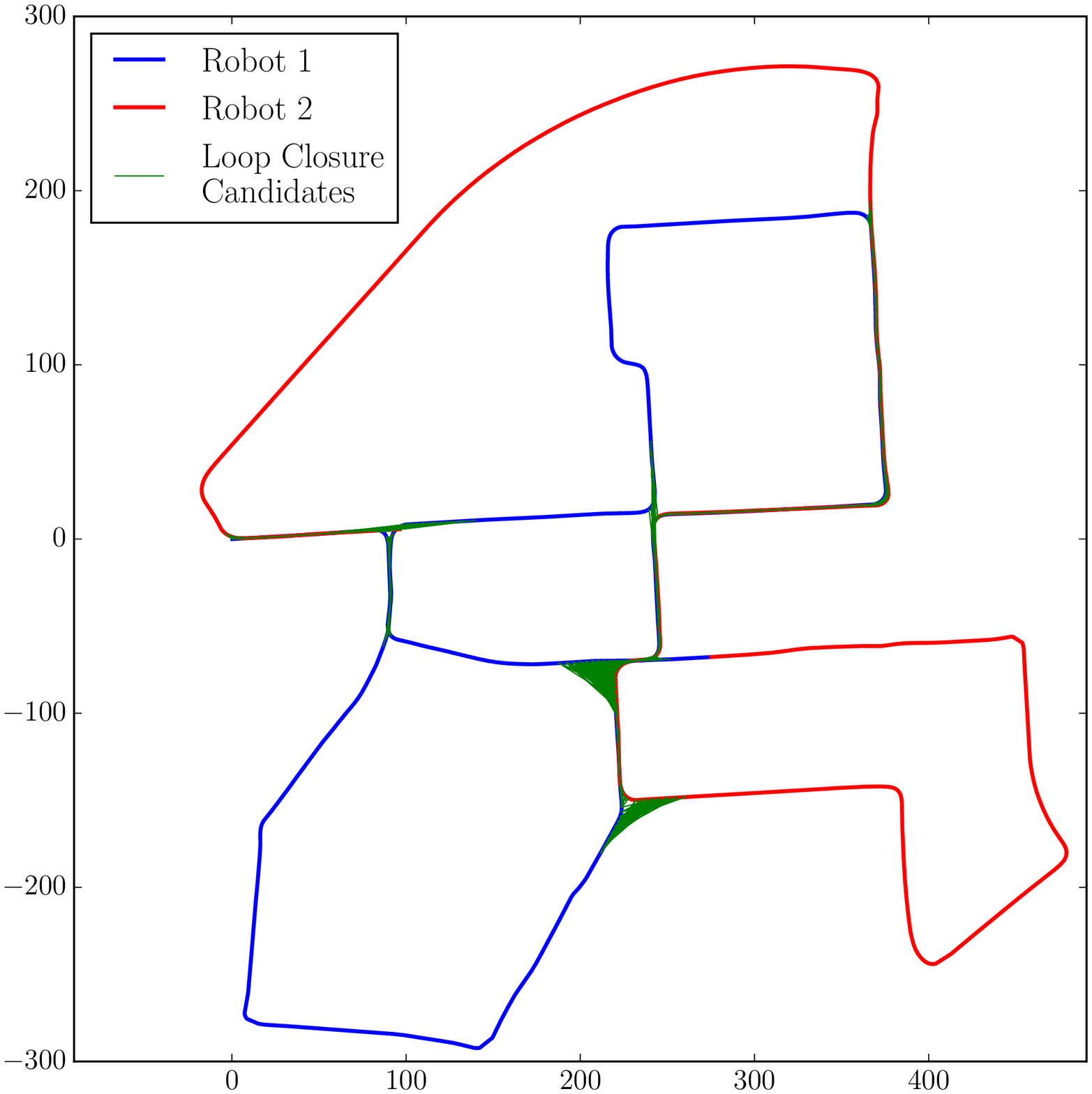}
  \caption{\small Sequence 0}
  \label{fig:kitti_trajectory_00}
\end{subfigure}
\\[0.4cm]
\begin{subfigure}{0.5\textwidth}
  \centering
  \includegraphics[scale=0.33]{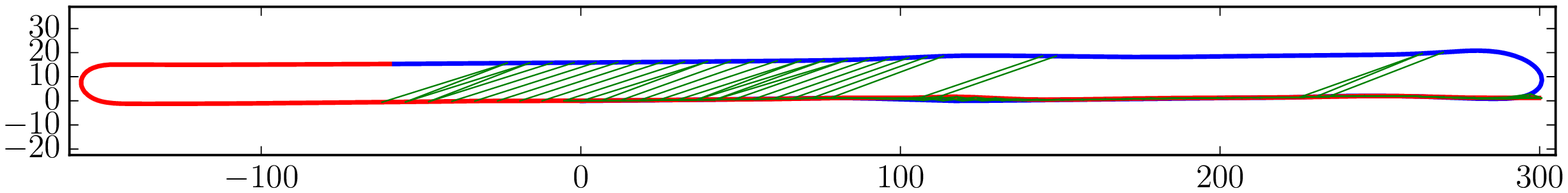}
  \caption{\small Sequence 6}
  \label{fig:kitti_trajectory_06}
\end{subfigure}
  
  \caption{\small Ground truth for KITTI odometry dataset sequences 0 and 6 with parameters $f$ = 2 Hz,  $\eta=0.4$, $d_\text{max} = 30\text{m}$ between robot 1 (blue) and robot 2's (red) for edges (green). The edges  and weights formed with ORB feature counts produce the exchange graph $\Gex$.}    
  \label{fig:kitti_trajectory}
  \vspace{0.2cm}
\end{figure}

\begin{figure*}[t]
\begin{subfigure}[t]{0.32\textwidth}
  \centering
  \includegraphics[width=\textwidth]{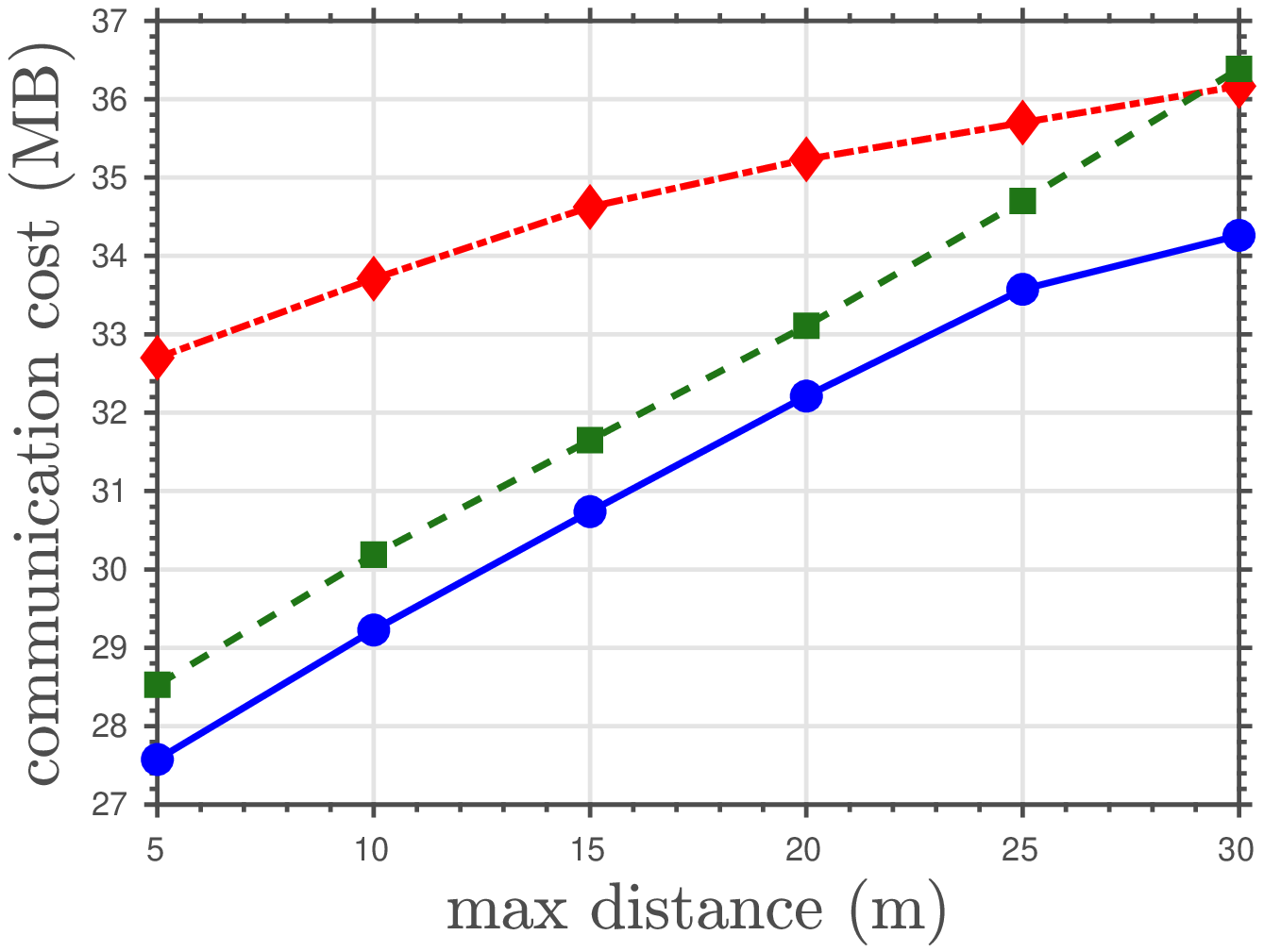}
  \caption{\small Sequence 0, $d_{\text{max}}$ threshold}
  \label{fig:radius_seq0}
\end{subfigure}
~
\begin{subfigure}[t]{0.32\textwidth}
  \centering
  \includegraphics[width=\textwidth]{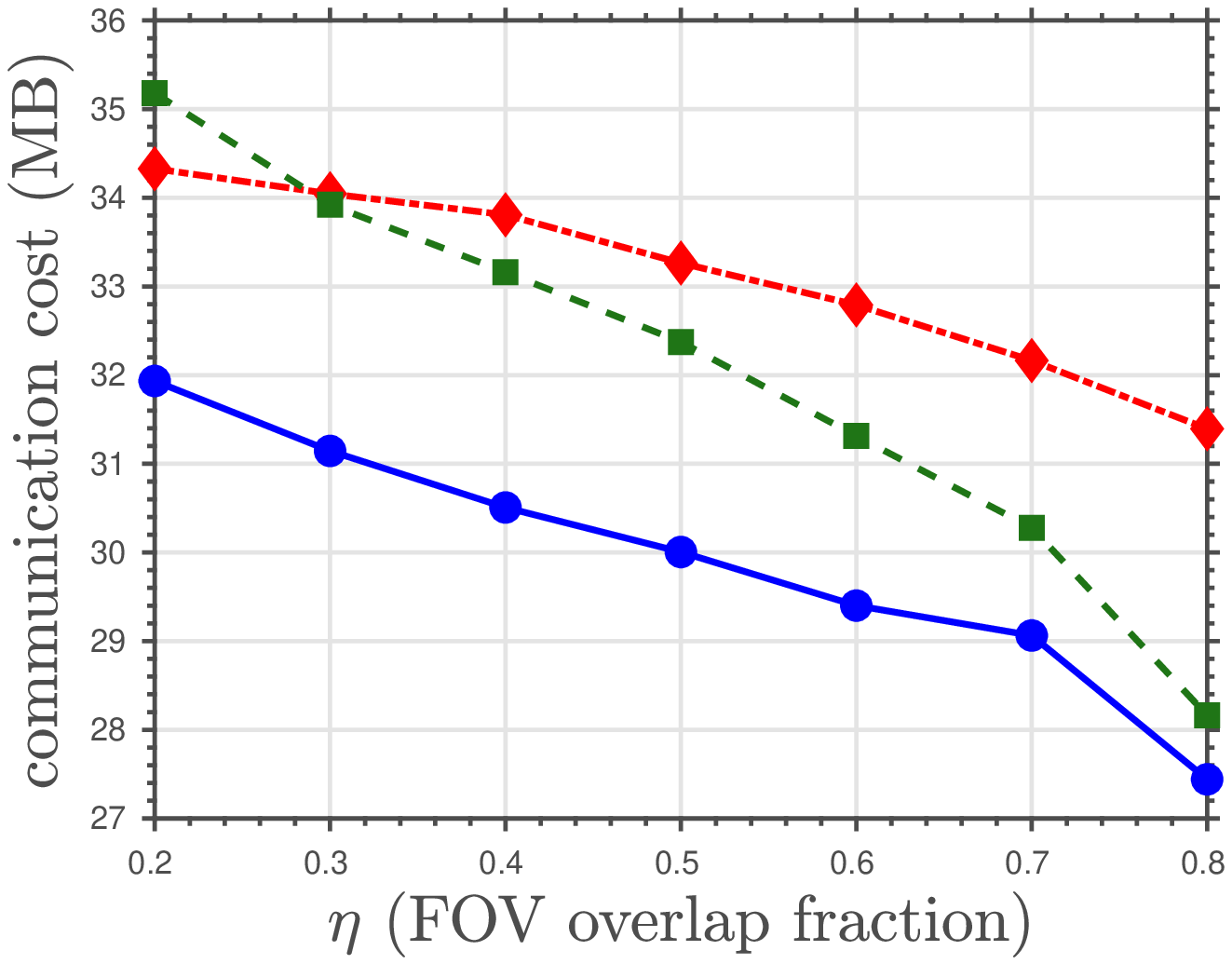}
  \caption{\small Sequence 0, $\eta$ threshold}
  \label{fig:fov_seq0}
\end{subfigure}
~
\begin{subfigure}[t]{0.32\textwidth}
  \centering
  \includegraphics[width=\textwidth]{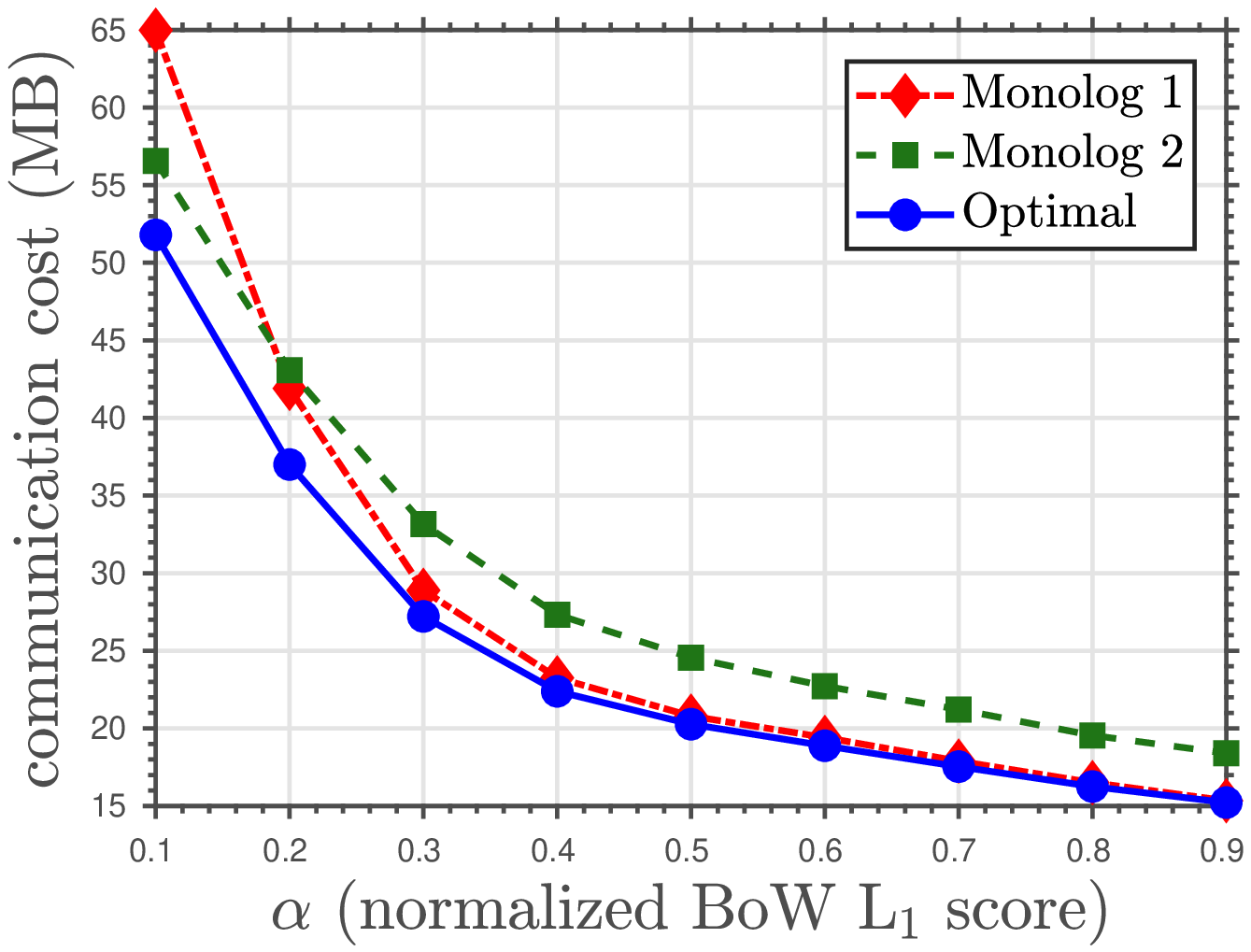}
  \caption{\small Sequence 0, $\alpha$ threshold}
  \label{fig:dbow_seq0}
\end{subfigure}
\\[0.4cm]
\begin{subfigure}[t]{0.32\textwidth}
  \centering
  \includegraphics[width=\textwidth]{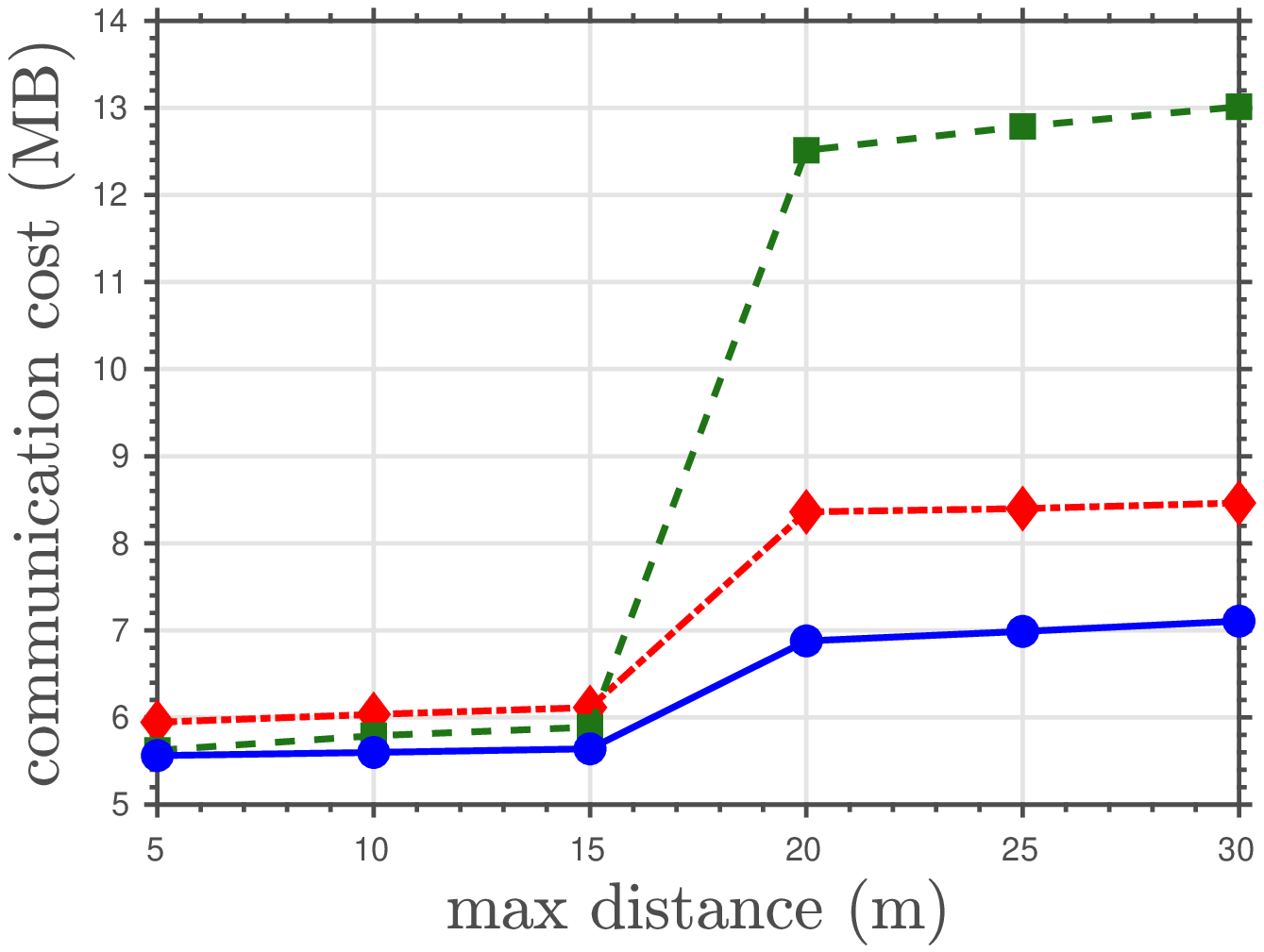}
  \caption{\small Sequence 6, $d_{\text{max}}$ threshold}
  \label{fig:radius_seq6}
\end{subfigure}
~
\begin{subfigure}[t]{0.32\textwidth}
  \centering
  \includegraphics[width=\textwidth]{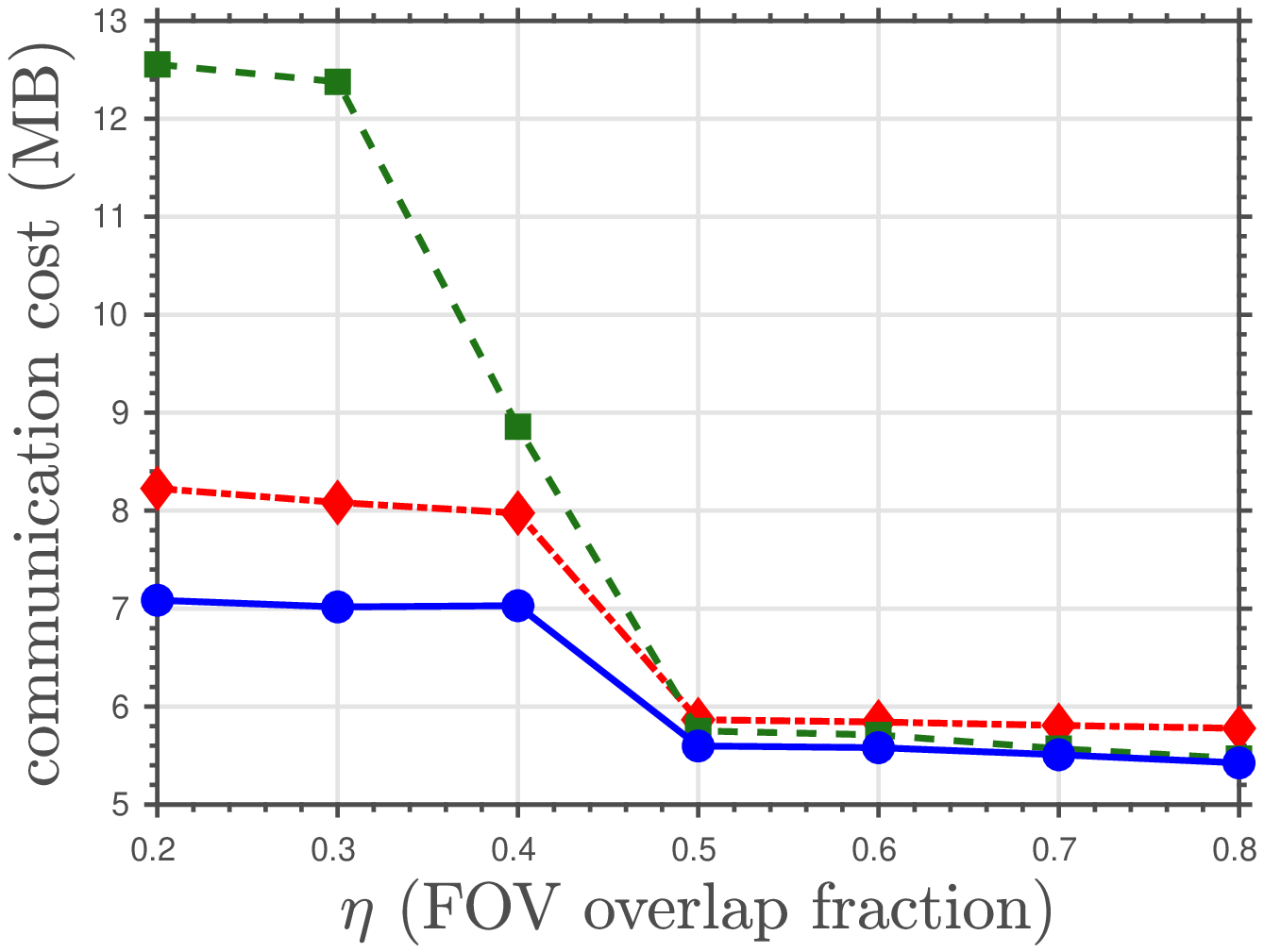}
  \caption{\small Sequence 6, $\eta$ threshold}
  \label{fig:fov_seq6}
\end{subfigure}
~
\begin{subfigure}[t]{0.32\textwidth}
  \centering
  \includegraphics[width=\textwidth]{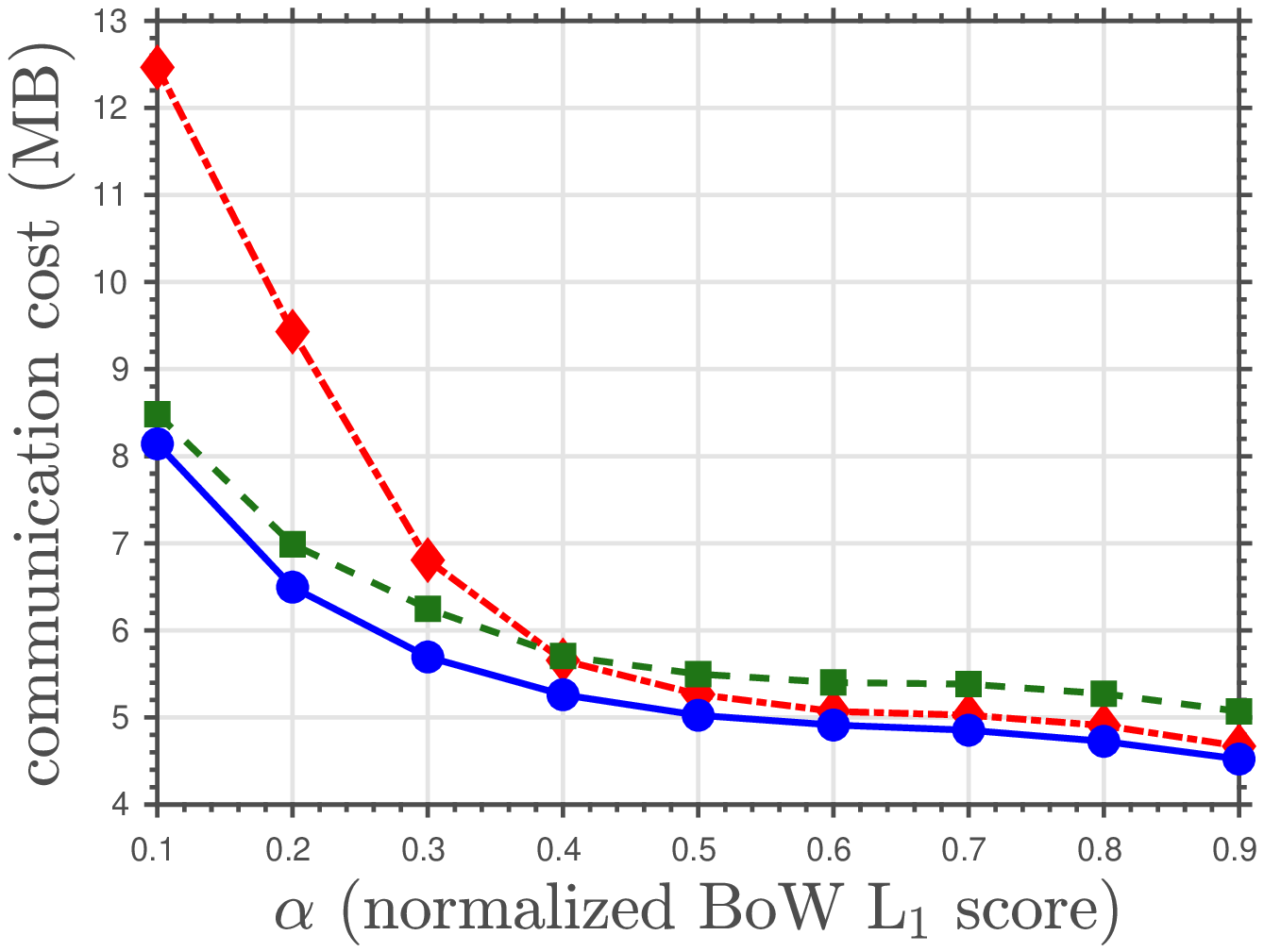}
  \caption{\small Sequence 6, $\alpha$ threshold}
  \label{fig:dbow_seq6}
\end{subfigure}
  \caption{\small \small Communication cost for KITTI odometry sequences 0 and 6
	with $k_F$ = 100, $f$ = 10 Hz, and varying $d_\text{max}$, $\eta$, or $\alpha$ (x-axes). Figures \ref{fig:radius_seq0} and \ref{fig:radius_seq6} form the candidate edge graph $\Lcal$ using maximum Euclidean distance $d_{\text{max}}$ between poses from KITTI groundtruth data, whereas Figures \ref{fig:fov_seq0} and \ref{fig:fov_seq6} use the fraction of overlapping fields  of view $\eta$ to determine candidates. Figures \ref{fig:dbow_seq0} and \ref{fig:dbow_seq6} form the candidate graph using the candidates with the 2 best DBoW2 scores greater than threshold $\alpha$.}
  \label{fig:kitti_params}
\end{figure*}

\subsection{Place Recognition Experiments}
\label{sec:dbow2}
An appearance-based place recognition system like DBoW2 \cite{GalvezTRO12} can
also be used for inter-robot loop closure detection and to 
generate the candidate edge set $\Lcal$. In a situation where robots do not have an accurate 
estimate of the transformation between their trajectories' frames of reference, place recognition 
must be leveraged instead of viewpoint proximity to find potential loop closures. To facilitate 
place recognition, DBoW2 only needs the ``word'' in the bag-of-words vocabulary describing each 
ORB feature~\cite{GalvezTRO12}. This word can typically be described in 3 or fewer bytes, which is less than one 
tenth of the size the standard 32 byte BRIEF descriptor used in ORB. Thus, an inexpensive exchange 
of vocabulary vectors (i.e., metadata) allows robots to search for
promising candidates and form an exchange graph $\Gex$ for ODEP.

In our experiments, we trained a DBoW2 vocabulary with parameters $k_w =10$ and
$L_w=6$ on ORB
features from 5 KITTI odometry benchmark
sequences. The two candidate edges with the highest normalized score exceeding
threshold $\alpha$ \cite{GalvezTRO12} were used to form candidate edges in
$\Lcal$. Communication savings from ODEP 
instances produced with KITTI odometry sequences 0 and 6 are displayed in Figures \ref{fig:dbow_seq0} 
and \ref{fig:dbow_seq6}. Although the structure of exchange graphs resulting from appearance based 
methods were very different from the geometric methods of the previous section, the cost of communication can still be significantly reduced using our method. For low $\alpha$ thresholds, the optimal policy affords 
significant savings of around 5-10 MB over the monolog policies in sequence 0 (Figure \ref{fig:dbow_seq0}). 
In the smaller sequence 6, the net communication savings are smaller because there are fewer candidates, 
but the optimal policy is still almost 10\% more efficient than the best monolog policies at $\alpha$ 
thresholds in the range of 0.2 to 0.4. It is also worth noting that relying on a single communication 
direction (i.e. using only one fixed monolog policy) throughout a mission is a poor communication heuristic 
that could produce arbitrarily bad and inconsistent results. This is illustrated in Figures \ref{fig:dbow_seq6} 
and \ref{fig:fov_seq0} where depending on the value of $\alpha$ or $\eta$, the better choice of monolog and its 
performance penalty relative to the optimal policy changes drastically. 

\section{Conclusion and Future Work}
\label{sec:conclusion}
Given the limitations of onboard resources, it is essential that robots communicate wisely.
State-of-the-art techniques often have to sacrifice content by down-sampling
the exchanged data, e.g.,
\cite{dong2015distributed,indelman2014multi,indelman2016incremental,cieslewski2017efficient}. This comes at
the risk of losing potential valuable inter-robot loop closures which are the
essence of cooperative localization and mapping.
This paper addressed this challenge by investigating the logistical aspect of sensory data sharing
in distributed CSLAM front-ends. First, we formalized
the optimal data exchange problem that encompasses a wide range of sensing
modalities (e.g., vision, 2D and 3D lasers). This led to a resource-efficient
and provably lossless (i.e., ensuring a complete search) 
communication planning framework. The proposed
framework takes into account both the quantity of exchanged data, and the
resulting division of labor induced by the executed exchange policy. This allows
us to design efficient communication plans while distributing the induced
workload based on, for example, the distribution of computational resources among
robots. Additionally, ODEP can seamlessly incorporate privacy and
security constraints through the concept of exchange inertia and dynamic
pricing schemes.
Our approach benefits greatly from several fundamental results in graph
theory and combinatorial optimization. In particular, these results lead to a
fast and provably tight LP relaxation scheme to find the globally optimal exchange
polices. In addition, our theoretical analysis characterized the necessary and sufficient
conditions under which simpler unidirectional exchange policies are optimal.
Finally, we experimentally validated geometric and appearance-based realizations of the 
proposed framework using the KITTI odometry benchmark datasets.

In retrospect, several crucial insights played major roles in the success of our
approach. First and foremost, identifying plausible inter-robot loop-closure
candidates \emph{before} transmitting the bulk of sensory data is what makes
communication planning possible. Forming the exchange graph and exploiting its
unique structure (topology and the vertex/edge weights) allowed us to identify
more efficient, yet lossless, exchange policies---often emerging as natural
dialogs. Although this requires exchanging ``metadata'', the incurred cost is
often not comparable to the that of the actual data exchange. For example, 
visual place recognition systems like DBoW2 form loop closure candidates 
with sparse feature vectors that use an order of magnitude less data than the 
full descriptors used for subsequent loop closure verification. 
In our experiments where robots found candidate edges by exchanging pose graphs,
poses are described by $\mathrm{SE}(2)$ or $\mathrm{SE}(3)$ objects that are
much smaller than hundreds of visual descriptors.  Furthermore, we exploited the
sparsity pattern of the graph in our implementation to solve the resulting LP
even faster.

This paper provides a solid foundation for optimal communication planning in
distributed CSLAM front-ends. Our approach is able to find the optimal
exchange policy between a pair of robots during pairwise encounters. $n$-way ($n > 2$) scan
exchange problems naturally arise in robotic networks with denser communication
graphs.  Although the proposed approach can
still be used in these cases, it may not necessarily lead to the optimal
strategy.  Addressing the sensory data exchange between more than two
robots requires exploring new mechanisms such as data caching and routing. The optimal
$n$-way data exchange problem is our next challenge.

\section*{Acknowledgement}
This work was supported in part by the NASA Convergent Aeronautics
Solutions project Design Environment for Novel Vertical Lift Vehicles
(DELIVER) and by the Northrop Grumman Corporation.
The authors would like to thank Sergii Iglin for the Graph Theory
Toolbox\footnote{\url{https://goo.gl/XJDwUf}},
the team behind the KITTI dataset \cite{Geiger2013IJRR}, and Noam Buckman of LIDS
for assistance with experiments. 

\bibliographystyle{IEEEtran} 
\bibliography{scanexchange} 
\end{document}

%% file: packages.tex
\usepackage{amsmath,amssymb,amsthm}
\usepackage[table]{xcolor}
\usepackage{caption}
\usepackage{subcaption}
\usepackage[]{mathtools}
\usepackage{bbm}
\usepackage{mdframed}
\usepackage[sort,compress]{cite}
\usepackage{tcolorbox}
\usepackage{algorithmicx}
\usepackage{algorithm}
\usepackage[noend]{algpseudocode}
\algnewcommand\algorithmicinput{\textbf{Input:}}
\algnewcommand\INPUT{\item[\algorithmicinput]}
\usepackage{moreverb,url}
\usepackage[bookmarksopen,bookmarksnumbered,colorlinks=true, linkcolor=black, citecolor = black, urlcolor = black,filecolor=black , pagebackref=false,hypertexnames=false, plainpages=false, pdfpagelabels ]{hyperref}
\usepackage{icomma} 
\usepackage{tikz} 
\usepackage{tkz-graph}
\usepackage{tkz-berge}
	\usetikzlibrary{shapes,snakes,arrows,shapes.geometric,positioning}
	\usetikzlibrary{intersections,patterns,shapes.misc}
	\usetikzlibrary{decorations.pathmorphing}
	\tikzstyle{block} = [rectangle, rounded corners, minimum width=3cm, minimum height=1cm,text centered, draw=black, fill=red!30]
	\tikzstyle{new} = [rectangle, rounded corners, minimum width=1cm, minimum
	height=1cm,text centered, draw=black, fill=blue!10!white, dashed]
	\tikzstyle{arrow} = [thick,->,>=stealth]
	\usetikzlibrary{calc, quotes}
\usepackage{pgfplots}
\usepgfplotslibrary{fillbetween}
\usepackage[perpage]{footmisc}
\usepackage{soul}
\usepackage{fontawesome}
\DeclareFontFamily{OT1}{pzc}{}
\DeclareFontShape{OT1}{pzc}{m}{it}{<-> s * [1.200] pzcmi7t}{}
\DeclareMathAlphabet{\mathpzc}{OT1}{pzc}{m}{it}

\AtBeginDocument{%
  \DeclareMathAlphabet\PazoBB{U}{fplmbb}{m}{n}%
}

\usepackage{dsfont}
\newtheorem{theorem}{Theorem}
\newtheorem{lemma}{Lemma}

\newtheorem{corollary}{Corollary}

\newtheorem{remark}{Remark}

\newtheorem{problem}{Problem}

\newtheorem{definition}{Definition}
\algrenewcommand\textproc{}%

\let\labelindent\relax

\let\oldthempfootnote\thempfootnote
\def\thempfootnote{\text{\oldthempfootnote}}
\usepackage{enumitem}

\setlist[itemize]{leftmargin=*}
\setlist[enumerate]{leftmargin=*}

\colorlet{lightgray}{green!4}

%% file: macros.tex
\newcommand{\GG}{\mathcal{G}}

\newcommand{\VV}{\mathcal{V}}

\newcommand{\zero}{\mathbf{0}}
\newcommand{\ones}{\mathbf{1}}

\newcommand{\ppp}{\boldsymbol\pi}
\newcommand{\AAA}{\mathbf{A}}
\newcommand{\AAAtop}{\AAA^{\hspace{-0.1cm}\top}}

\newcommand{\Acal}{\mathcal{A}}

\newcommand{\Scal}{\mathcal{S}}
\newcommand{\Ncal}{\mathcal{N}}
\newcommand{\Bcal}{\mathcal{B}}
\newcommand{\Mcal}{\mathcal{M}}

\newcommand{\Fcal}{\mathcal{F}}

\DeclareMathAlphabet\mathbfcal{OMS}{cmsy}{b}{n}

\newcommand{\Rset}{\mathbb{R}}

\newcommand{\Vcal}{\mathcal{V}}
\newcommand{\Lcal}{\mathcal{L}}

\newcommand{\fcomm}{f_\text{\tiny\faWifi}}

\newcommand{\fcomp}{f_\text{\tiny\faBalanceScale}}
\newcommand{\Gex}{\GG}

	\newcommand\DoubleLine[5][4pt]{%
	  \path(#2)--(#3)coordinate[at start](h1)coordinate[at end](h2);
	  \draw[<-,line width=0.25mm,black] ($(h1)!#1!90:(h2)$)  to ["#4"]   ($(h2)!#1!-90:(h1)$);
	  \draw[->,line width=0.25mm,black] ($(h1)!#1!-90:(h2)$) to ["#5" '] ($(h2)!#1!90:(h1)$);
	}
	\newcommand\DoubleLineOneDirectional[5][4pt]{%
	  \path(#2)--(#3)coordinate[at start](h1)coordinate[at end](h2);
	  \draw[->,line width=0.75mm,black] ($(h1)!#1!90:(h2)$)  to ["#4"]
	  ($(h2)!#1!-90:(h1)$);
	  \draw[->,very thick,densely dashed,black!50!white, bend right=15] ($(h1)!#1!-100:(h2)$) to
	  ["#5" '] ($(h2)!#1!100:(h1)$);
	}
	\newcommand\xDoubleDashedLine[5][4pt]{%
	  \path(#2)--(#3)coordinate[at start](h1)coordinate[at end](h2);
	  \draw[->,very thick,densely dashed,black!50!white,bend left=90] ($(h1)!#1!90:(h2)$)  to ["#4"]
	  ($(h2)!#1!-90:(h1)$);
	  \draw[<-,very thick,densely dashed,black!50!white, bend right=80] ($(h1)!#1!-100:(h2)$) to
	  ["#5" '] ($(h2)!#1!100:(h1)$);
	}
	\newcommand\DoubleDashedLine[5][4pt]{%
	  \path(#2)--(#3)coordinate[at start](h1)coordinate[at end](h2);
	  \draw[<-,very thick,black!50!white,densely
	  dashed,sloped,anchor=south] ($(h1)!#1!90:(h2)$)  to ["#4"]   ($(h2)!#1!-90:(h1)$);
	  \draw[->,very thick,  black!50!white,densely dashed,sloped,anchor=south] ($(h1)!#1!-90:(h2)$) to ["#5" '] ($(h2)!#1!90:(h1)$);
	}
	\newcommand\DoubleDashedLineRev[5][4pt]{%
	  \path(#2)--(#3)coordinate[at start](h1)coordinate[at end](h2);
	  \draw[<-,very thick,  black!50!white,densely dashed,sloped,anchor=south] ($(h1)!#1!-90:(h2)$) to ["#5" '] ($(h2)!#1!90:(h1)$);
	  \draw[->,very thick,black!50!white,densely dashed,sloped,anchor=south] ($(h1)!#1!90:(h2)$)  to ["#4"]   ($(h2)!#1!-90:(h1)$);
	}

\definecolor{kkGreen}{RGB}{201,232,206}
\definecolor{kkRed}{RGB}{255,196,215}
\definecolor{kkBlue}{RGB}{214,226,255}